\documentclass{article}
\usepackage[final, nonatbib]{eiml_style_2025}
\usepackage[numbers]{natbib}    
\usepackage{graphicx}           
\usepackage{amsmath}            
\usepackage{amssymb}            
\usepackage{amsthm}            
\usepackage{mathtools}          
\usepackage{bbm}                
\usepackage{wrapfig}            
\usepackage[table]{xcolor}      
\usepackage{subcaption}         
\usepackage{enumitem}           
\usepackage{url}
\usepackage{hyperref}
\usepackage{thm-restate}

\title{
    Is BatchEnsemble a Single Model? \\ 
    On Calibration and Diversity of Efficient 
    Ensembles
}
\author{
    Anton Zamyatin \\
    TU Wien \\
    \texttt{anton.zamyatin@tuwien.ac.at}
    \And
    Patrick Indri \\ 
    TU Wien \\
    \texttt{patrick.indri@tuwien.ac.at}
    \And
    Sagar Malhotra \\
    TU Wien \\
    \texttt{sagar.malhotra@tuwien.ac.at}
    \And
    Thomas Gärtner \\
    TU Wien \\
    \texttt{thomas.gaertner@tuwien.ac.at}
}
\date{\today}

\begin{document}

\maketitle
\begin{abstract}
In resource-constrained and low-latency settings, uncertainty estimates must be efficiently obtained.
Deep Ensembles provide robust epistemic uncertainty (EU) but require training multiple full-size models.
BatchEnsemble aims to deliver ensemble-like EU at far lower parameter and memory cost by applying learned rank-1 perturbations to a shared base network.
We show that BatchEnsemble not only underperforms Deep Ensembles but closely tracks a single model baseline in terms of accuracy, calibration and out-of-distribution (OOD) detection on CIFAR-10/10C/SVHN. 
A controlled study on MNIST finds members are near-identical in function and parameter space, indicating limited capacity to realize distinct predictive modes.
Thus, BatchEnsemble behaves more like a single model than a true ensemble.
\end{abstract}

 \section{Introduction}

Modern neural networks are increasingly used in safety-critical settings, where reliable uncertainty estimates are essential. 
Ensembles can be used to improve both accuracy and calibration \cite{lakshminarayanan2017deepEnsembles}.
Ensemble methods train multiple models independently and aggregate their predictions to capture epistemic uncertainty (EU), that is, the uncertainty in the model parameters that arises from training.
Techniques such as Deep Ensembles \citep{lakshminarayanan2017deepEnsembles} can be used to evaluate EU but require to train a neural network many times, which is often prohibitively expensive.
BatchEnsemble \cite{wen2020batchensemble} address this by (i) applying rank-1 perturbations to a single shared weight matrix to form distinct members and (ii) vectorizing inference so all members run in a single forward pass.
Despite their popularity, there is no systematic comparison of BatchEnsemble with standard Deep Ensembles; this paper fills that gap and investigates the shortcomings of BatchEnsemble.
We compare BatchEnsemble to Deep Ensembles as well as MC Dropout \cite{gal2016MCDropout}, and structure the rest of the paper as follows.
In Section~\ref{sec:Background}, we provide some preliminary background on Deep Ensembles, BatchEnsembles, notions of ensemble diversity and EU in ensemble methods.
In Section~\ref{sec:zero-measure}, we theoretically highlight that BatchEnsemble can only explore a significantly restricted portion of the parameter space, compared to Deep Ensemble.
In Section~\ref{sec:batchensemble-underperforms} we empirically show that BatchEnsemble underperform both Deep Ensemble and MC Dropout in terms of predictive performance across various metrics.
In Section~\ref{sec:batchensemble-diversity} we show that BatchEnsembles learned on MNIST \cite{mu2019mnistc} data are essentially equivalent to a single model, further confirming that the theoretical observations of restricted parameter space also transfers to restricted solution space in real-world settings.

\section{Background and Related Work}
\label{sec:Background}
\paragraph{Deep Ensemble.}
Deep Ensembles \citep{lakshminarayanan2017deepEnsembles} approximate the posterior distribution $p(y\mid x,D)$ of observing label $y$ for an unseen input $x$ given the training data $D$ by uniformly averaging the outputs of $k$ independently trained models:
\begin{equation}
    \hat{p}(y \mid x, D) = \frac{1}{k}\sum^k_{i=1}p(y \mid x, \theta_i),
\end{equation}
where $\theta_i$ are parameters obtained from different random initialization and training trajectories.
Deep Ensembles provide robust epistemic uncertainty (EU) under distribution shift and outperform scalable Bayesian approximations such as stochastic variational inference (SVI) \cite{ovadia2019predictiveUncertaintyDataShift}.
However, the computational and memory cost of Deep Ensembles scales linearly with $k$, which makes them impractical for low-latency or resource constrained settings such as high-throughput screening, real-time systems, or large-scale models.

\paragraph{BatchEnsemble.}
To mitigate the linear scaling in computational/memory cost of Deep Ensembles,
\citet{wen2020batchensemble} introduced BatchEnsemble.
BatchEnsemble promises to deliver ensemble-level EU at (nearly) single model cost.
To this end, BatchEnsemble embeds an ensemble within one network by applying learned low-rank, per-member perturbations to a shared weight matrix.
Let $W \in \mathbb{R}^{m \times n}$ denote a layer's shared weight and $r_i\!\in\!\mathbb{R}^m$, $s_i\!\in\!\mathbb{R}^n$ be member-specific learned vectors.
Then, BatchEnsemble parametrizes the weight matrix $W_i$ of each member $i=1, \cdots, k$ as
\begin{equation}
\label{eq:batch-ensemble}
    W_i = W \circ r_i s_i^\top,
\end{equation}
where $\circ$ denotes the Hadamard product and each $r_is_i^\top$ is a rank-1 matrix.
BatchEnsemble thus reduces the per-layer parameters from $O(kmn)$ for independent members obtained using Deep Ensembles, to $O(mn + k(m+n))$.
This leads to a substantial parameter reduction, especially for large matrices where $k\!\ll\!\min(m,n)$.
The forward pass for a single input $x \in \mathbb{R}^n$ can be written as
\begin{equation}
\label{eq:be-forward}
    W_ix = (W \circ r_i s_i^\top) = W(s_i \circ x) \circ r_i,
\end{equation}
with member-specific biases omitted for brevity.
This formulation admits a vectorized implementation (Appendix~\ref{appendix:batchensemble-vectorization}) that computes the prediction of all $k$ members in a single \emph{batched} forward pass---hence ``BatchEnsemble''---enabling within-device parallelism and high-throughput inference.
By contrast the popular ensemble alternative MC Dropout \cite{gal2016MCDropout} requires multiple stochastic forward passes to obtain an ensemble of predictions.

\paragraph{Ensemble Diversity and EU.}
A leading explanation for the success of Deep Ensembles is that independently trained members explore distinct solutions not only in parameter space but also in function space \cite{fort2019deep}.
In contrast, scalable Bayesian methods such as mean-field VI \cite{graves2011practical}, Bayes by Backprop \cite{blundell2015weight}, or MC Dropout \cite{gal2016MCDropout} approximate uncertainty by placing a simple distribution (e.g., diagonal Gaussian or dropout mask) centered around a single optimum.
Consequently, efficient ensembling methods should be designed such that they are able to replicate this effect by capturing distinct predictive modes, not just local perturbations around a single optimum.

\paragraph{This work.}
We compare BatchEnsemble to Deep Ensembles and MC Dropout on CIFAR-10 across in-distribution, distribution shift, and out-of-distribution (OOD) settings (Section~\ref{sec:batchensemble-underperforms}; protocols and metrics in Appendix~\ref{appendix:experimental-details-cifar-10} and \ref{appendix:metrics}).
BatchEnsemble does not deliver ensemble-level gains, yielding low epistemic uncertainty and closely matching the single-model baseline in terms of predictive performance, calibration, and OOD detection.
A diagnostic MNIST study further reveals BatchEnsemble members are nearly identical in parameter space and function space, implying that limited capacity to realize multiple predictive modes may underlie BatchEnsemble's limited EU.

\section{Multiplicative Rank-1 Perturbations Have Measure Zero}
\label{sec:zero-measure}

We begin by making some simple but pertinent theoretical observations that show that the parameter space expressible by BatchEnsembles is vanishingly small compared to Deep Ensembles. 

\begin{restatable}{observation}{setinclusion}
\label{obs2}
Consider ensembles with $k$ members.
For BatchEnsemble, consider the members as parameterized by  $r_i \in \mathbb{R}^m$, $s_i\in \mathbb{R}^n$, and the shared weight matrix $W$.
Let $\mathcal{W}_\text{BE} =  \{ \left( W \circ r_1s_1^\top, \dots, W \circ r_k s_k^\top \right) \mid r_i \in \mathbb{R}^m, s_i\in \mathbb{R}^n, i \in [k], W \in \mathbb{R}^{m\times n}\}$ be the set of weights achievable with BatchEnsemble.
Let $\mathcal{W}_\text{DE} = \{ \left(W_1, \dots, W_k \right) \mid W_i \in \mathbb{R}^{m\times n}, i\in [k] \}$ be the set of weights achievable with Deep Ensembles, where each ensemble member has an independent weight matrix layer $W_i$ of size $\mathbb{R}^{m\times n}$.
Then, $\mathcal{W}_\text{BE} \subsetneq \mathcal{W}_\text{DE}$. Moreover, $\mathcal{W}_\text{BE}$ has measure zero in $\mathcal{W}_\text{DE}$.
\end{restatable}

A proof of Observation \ref{obs2} and examples of weight matrices unrepresented weight matrices are provided in Appendix~\ref{appendix:proofs}. 
This result shows that the weights reachable by BatchEnsemble’s rank-1 multiplicative perturbations constitute a strict, measure-zero subset of those reachable by independently trained members, i.e., Deep Ensembles.
This limited expressivity suggests that BatchEnsemble cannot capture the same degree of diversity as a Deep Ensemble.
In the next sections, we test this hypothesis empirically on real data.

\section{BatchEnsemble Underperforms Deep Ensemble}
\label{sec:batchensemble-underperforms}
We compare a single model, MC Dropout, Deep Ensembles, and BatchEnsemble on image classification, using the same architecture (ResNet-18 \cite{he2016resnet}) trained for 75 epochs on CIFAR-10 \cite{krizhevsky2009cifar}.
Following the protocol of \citet{ovadia2019predictiveUncertaintyDataShift}, we evaluate the models on three benchmarks:
\begin{itemize}[noitemsep, topsep=1pt]
    \item CIFAR-10 test set to measure in-distribution performance,
    \item CIFAR-10C \cite{hendrycks2019benchmarkingRobustness} at corruption severity 5 to assess robustness to distributional shift,
    \item SVHN \cite{netzer2011svhn} to test OOD detection.
\end{itemize}
For the in-distribution and distribution shift benchmarks, we report accuracy (Acc), negative log-likelihood (NLL), and expected calibration error (ECE).
We treat OOD detection as a binary classification task and report the area under the precision–recall curve (AUPR), the area under the ROC curve (AUROC), and the false positive rate at 95\% TPR (FPR95).
For each metric, we indicate whether lower (↓) or higher (↑) is better.
To further characterize uncertainty we compute two diagnostic metrics:
the predictive entropy (H) for total uncertainty and the Jensen–Shannon divergence (JSD) as a proxy for EU.
We report the mean and 95\% confidence intervals over five random seeds, computed using the standard error of the mean and the $t$-distribution with 4 degrees of freedom.
The results are summarized in Tables~\ref{tab:ind}--\ref{tab:ood}. 
Details on the metrics used are given in Appendix~\ref{appendix:metrics}.
A description of the training procedure and experimental considerations can be found in Appendix~\ref{appendix:experimental-details-cifar-10}.


\begin{table}[!h]
    \centering
    \footnotesize
    \caption{In-Distribution Benchmark (CIFAR-10)}
    \begin{tabular}{c|ccc|cc}
        \hline
        \textbf{Method} & \textbf{Acc↑} & \textbf{NLL↓} & \textbf{ECE↓} & \textbf{H} & \textbf{JSD} \\ \hline
        Single & \cellcolor{blue!10}0.941 ±0.004 & 0.237 ±0.003 & 0.034 ±0.003 & 0.073 ±0.005 & N/A \\ 
        MC Dropout & 0.936 ±0.001 & \cellcolor{blue!10}0.210 ±0.007 & \cellcolor{blue!10}0.021 ±0.002 & 0.120 ±0.005 & 0.022 ±0.001 \\ 
        Deep Ensemble & \cellcolor{blue!30}0.952 ±0.001 & \cellcolor{blue!30}0.152 ±0.003 & \cellcolor{blue!30}0.007 ±0.002 & 0.136 ±0.001 & 0.037 ±0.001 \\ 
        BatchEnsemble & 0.938 ±0.001 & 0.230 ±0.006 & 0.032 ±0.001 & 0.088 ±0.003 & 0.002 ±0.000 \\
        \hline
    \end{tabular}
    \label{tab:ind}
\end{table}

\begin{table}[!h]
    \centering
    \footnotesize
    \caption{Distribution Shift Benchmark (CIFAR-10C, shift intensity 5)}
    \begin{tabular}{c|ccc|cc}
        \hline
        \textbf{Method} & \textbf{Acc↑} & \textbf{NLL↓} & \textbf{ECE↓} & \textbf{H} & \textbf{JSD} \\ \hline
        Single & 0.558 ±0.008 & 2.545 ±0.142 & 0.323 ±0.014 & 0.327 ±0.026 & N/A \\ 
        MC Dropout &\cellcolor{blue!30}0.578 ±0.002 & \cellcolor{blue!10}2.054 ±0.084 & \cellcolor{blue!10}0.251 ±0.009 & 0.449 ±0.031 & 0.088 ±0.005 \\ 
        Deep Ensemble & \cellcolor{blue!10}0.575 ±0.006 & \cellcolor{blue!30}1.682 ±0.045 & \cellcolor{blue!30}0.206 ±0.009 & 0.593 ±0.018 & 0.176 ±0.005 \\ 
        BatchEnsemble & 0.547 ±0.015 & 2.357 ±0.164 & 0.308 ±0.012 & 0.400 ±0.011 & 0.008 ±0.002 \\
        \hline
    \end{tabular}
    \label{tab:shifted}
\end{table}

\begin{table}[!h]
    \centering
    \footnotesize
    \caption{OOD Benchmark (SVHN)}
    \begin{tabular}{c|ccc|cc}
        \hline
        \textbf{Method} & \textbf{AUPR↑} & \textbf{AUROC↑} & \textbf{FPR95↓} & \textbf{H} & \textbf{JSD} \\ \hline
        Single & \cellcolor{blue!10}0.938 ±0.017 & \cellcolor{blue!10}0.893 ±0.024 & 0.318 ±0.071 & 0.480 ±0.096 & N/A \\ 
        MC Dropout & 0.932 ±0.008 & 0.884 ±0.009 & 0.324 ±0.019 & 0.608 ±0.056 & 0.122 ±0.016 \\ 
        Deep Ensemble & \cellcolor{blue!30}0.953 ±0.006 & \cellcolor{blue!30}0.924 ±0.011 & \cellcolor{blue!30}0.199 ±0.024 & 0.829 ±0.067 & 0.277 ±0.032\\ 
        BatchEnsemble & 0.932 ±0.015 & 0.888 ±0.020 & \cellcolor{blue!10}0.297 ±0.061 & 0.495 ±0.100 & 0.010 ±0.006\\
        \hline
    \end{tabular}
    \label{tab:ood}
\end{table}

Across all benchmarks, Deep Ensembles achieve the strongest accuracy, calibration, and ID/OOD separation.
MC Dropout, while cheaper, still exhibits ensemble-like gains over a single model.
In contrast, BatchEnsemble yields only marginal improvements in calibration and its EU, measured by JSD, remains near-zero across all settings.
These observations suggest that the EU offered by BatchEnsemble is categorically different, as it is closer to a single model than to a true ensemble of different models.
In the next section, we therefore investigate whether limited diversity underlies this behavior in a controlled setting next.

\section{BatchEnsemble Is Not Diverse}
\label{sec:batchensemble-diversity}
The results provided in the previous section suggest that BatchEnsemble behaves more like a single model than an ensemble.
If this stems from insufficient member diversity, we should observe it in a controlled setting with minimal confounds.
We therefore analyze diversity in a simple 3-layer perceptron on MNIST with $k{=}4$ members (Appendix~\ref{appendix:experimental-details-mnist}), where perturbations apply to linear layers only.
This choice removes channel–rescaling effects in CNNs (Appendix~\ref{appendix:channel-rescaling-vs-kernel-perturbation}).
The architecture can therefore neither mask diversity in BatchEnsemble, nor induce it in Deep Ensembles, allowing for a fair comparison of ensemble diversity.

We assess diversity in both function and parameter space. 
Let $f_{\theta}:\mathcal{X}\!\to\!\mathcal{Y}$ denote a neural network predictor with weights $\theta$.
In function space, we use pairwise prediction disagreement on $\{x_i\}_{i=1}^N$:
\begin{equation*}
\mathrm{Disagreement}(f_{\theta_1},f_{\theta_2})=\tfrac{1}{N}\sum_{i=1}^N \mathbbm{1}\!\left[f_{\theta_1}(x_i)\neq f_{\theta_2}(x_i)\right],
\end{equation*}
In parameter space, we compare the (vectorized) weights via cosine similarity:
\begin{equation*}
\cos(\theta_1,\theta_2)=\frac{\langle\theta_1,\theta_2\rangle}{\|\theta_1\|\,\|\theta_2\|}.
\end{equation*}
Since BatchEnsemble defines members implicitly through input/output scalings per layer (Eq.~\ref{eq:be-forward}), we form explicit weight tensors using Eq.~\ref{eq:batch-ensemble} before computing cosine similarity.
We present the results of our experiments in Figure~\ref{fig:2}.

\begin{figure}[h!]
    \centering
    \begin{subfigure}{0.74\textwidth}
        \centering
        \includegraphics[width=\linewidth]{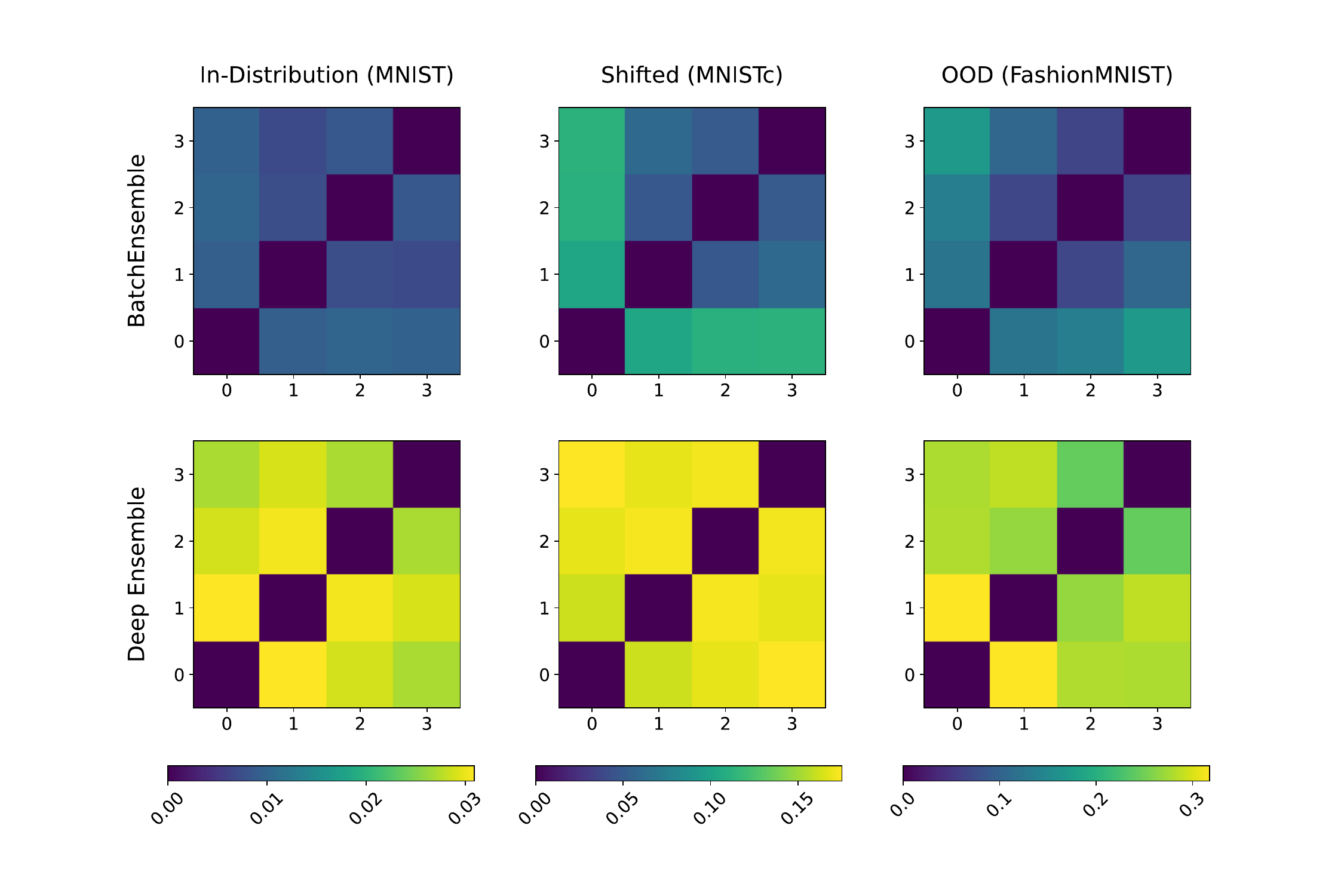}
        \subcaption{Disagreement across in-distribution, distribution shifted and OOD test sets.}
    	\label{fig:disagreement-matrices-be-vs-de}
    \end{subfigure}
    \hfill
    \begin{subfigure}{0.245\textwidth}
        \centering
        \includegraphics[width=\linewidth]{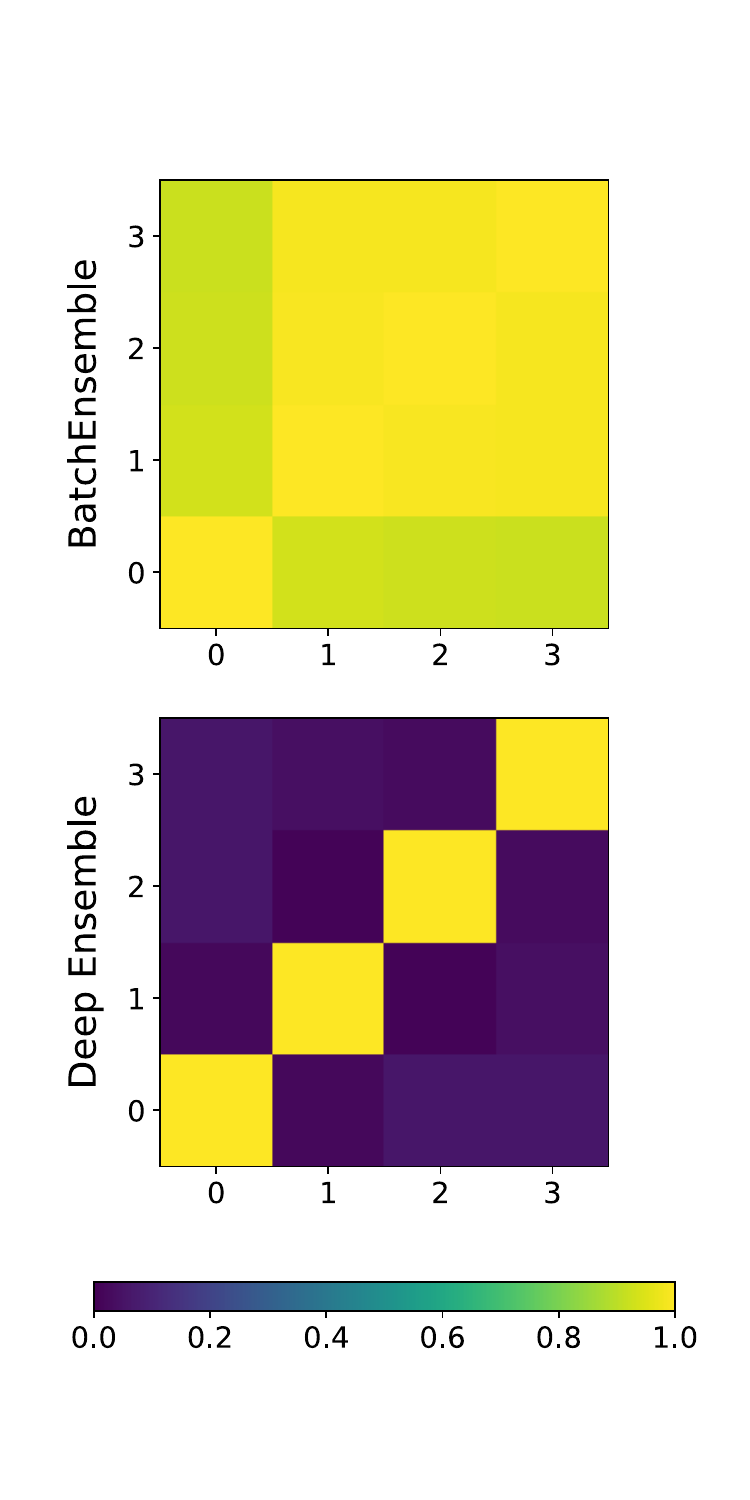}
        \subcaption{Cosine similarity}
    	\label{fig:cosine-similarity-matrices-be-vs-de}
    \end{subfigure}
    \caption{
        \textbf{BatchEnsemble lacks functional and parametric diversity (MNIST MLP).}
        (a) Pairwise prediction disagreement across ID, distribution-shifted, and OOD test sets (higher is more diverse).
        (b) Cosine similarity of members’ weights (higher is more similar).
        Deep Ensembles are diverse in both spaces; BatchEnsemble members cluster tightly.
    }\label{fig:2}
\end{figure}

BatchEnsemble exhibits minimal \emph{functional} diversity: pairwise prediction disagreement is near-zero across ID, shift, and OOD settings (Figure~\ref{fig:disagreement-matrices-be-vs-de}).
It also fails to diversify in \emph{parameter} space, with members showing cosine similarity $\approx 1$ (Figure~\ref{fig:cosine-similarity-matrices-be-vs-de}).
In contrast, Deep Ensembles display substantial predictive disagreement and low weight-space similarity, indicating distinct solutions.
A slight deviation of the shared base path (index~0) under shift/OOD does not translate into meaningful member diversity since the perturbed members~1–3 remain tightly coupled in both spaces.
Together, these results indicate limited capacity of BatchEnsemble to realize multiple predictive modes, consistent with its near-zero epistemic uncertainty in Section~\ref{sec:batchensemble-underperforms}.

\section{Conclusion}
BatchEnsemble is designed to provide ensemble-like uncertainty at near single-model cost, yet in our setting it failed to reproduce the behaviour of Deep Ensembles.
Across CIFAR-10/10C/SVHN, its predictive performance and calibration closely tracked a single model, and its epistemic uncertainty remained low.
A controlled MNIST study further revealed near-zero functional disagreement and near-unity weight-space similarity, indicating limited ability to realize distinct predictive modes.

To support a fuller assessment of efficient ensemble-like UQ methods, we recommend reporting uncertainty measures such as predictive entropy (H) and Jensen–Shannon divergence (JSD) alongside standard accuracy, calibration, and OOD metrics.
Simple diagnostics of functional and weight-space diversity, as used here, can additionally help determine whether a method expresses genuinely distinct predictive modes or merely produces correlated variants of a single model.

\section{Limitations \& Future Work}
Our findings differ from the stronger BatchEnsemble results reported by \citet{wen2020batchensemble}.
This discrepancy is largely attributable to scale \footnote{
    \citet{wen2020batchensemble} used wider and deeper networks trained for substantially longer (e.g., ResNet-32×4, 250–375 epochs), whereas our evaluation followed the lighter protocol of \citet{ovadia2019predictiveUncertaintyDataShift} with ResNet-18 trained for 75 epochs.
    The improved single-model performance reported in \citet{wen2020batchensemble} relative to the Ovadia benchmark suggests that much of the reported gain is driven by increased capacity rather than ensembling-specific effects.
    Unfortunately, their benchmark does not report uncertainty measures such as H or JSD and does not analyse functional or weight-space diversity.
}.
Although larger models can narrow the performance gap, scaling and ensembling offer fundamentally different advantages under equal parameter budgets \cite{lobacheva2020power}.
Thus, a natural next step is to measure BatchEnsemble’s  diversity and epistemic uncertainty under increased width, depth, and training time, extending this analysis to the exact settings used by \citet{wen2020batchensemble}. 
It would also be valuable to assess these behaviors across different architectures and tasks, 
including Transformers and regression settings, to test the generality of our observations. 
Finally, a theoretical account connecting empirical diversity limitations to the constraints imposed by BatchEnsemble’s rank-1 multiplicative parameterization may help clarify the fundamental constraints of the method.

\bibliographystyle{plainnat}
\bibliography{references}

\appendix

\section{BatchEnsemble Vectorization}
\label{appendix:batchensemble-vectorization}
To enable a vectorized forward pass across all $k$ ensemble members, \citet{wen2020batchensemble} constrain the mini-batch size $b$ to be divisible by $k$ such that each ensemble member processes a disjoint subset of $b/k$ samples.
The input matrix $X \in \mathbb{R}^{b \times n}$ is split accordingly, and the corresponding $k$ rank-1 scaling vectors $r_i \in \mathbb{R}^m$ and $s_i \in \mathbb{R}^n$ are tiled $b/k$ times to form the matrices $R \in \mathbb{R}^{b \times m}$ and $S \in \mathbb{R}^{b \times n}$, which contain the row and column scaling factors, respectively.

This allows the forward pass for the full mini-batch to be written as:
\begin{equation}
\label{eq:be-vectorized}
    \bigl( (S \circ X) W^\top \bigr) \circ R,
\end{equation}

where $W \in \mathbb{R}^{m \times n}$ is the shared weight matrix, and Hadamard products are applied row-wise.
While this formulation is computationally efficient, each ensemble member only sees a fraction $1/k$ of the data in a given iteration.

Alternatively, the mini-batch can be expanded such that each input sample is repeated $k$ times, once for each ensemble member.
This results in an input matrix $X \in \mathbb{R}^{bk \times n}$, where each original sample appears $k$ times.
The scaling matrices $R \in \mathbb{R}^{bk \times m}$ and $S \in \mathbb{R}^{bk \times n}$ are constructed by repeating each ensemble member’s vectors $r_i$ and $s_i$ over the entire batch.
This ensures that each ensemble member processes all data in every training iteration, at the cost of increased training time due to input duplication (see Appendix~\ref{appendix:computational-cost}).

\section{Proofs For Limited Expressivity Of BatchEnsemble}
\label{appendix:proofs}

\setinclusion*
\begin{proof}
  The inclusion trivially holds as any weight reached by BatchEnsemble is in $\mathbb{R}^{m\times n}$, and thus reachable by Deep Ensembles.
  Strictness can be shown as follows.
  Let $\oslash$ be the element-wise division.
  Consider two weight matrices $A,B\in \mathbb{R}^{m\times n}$ with non-zero entries and such that $Q := A \oslash B$ has rank strictly larger than 1.
  If $A$ and $B$ were to be representable by BatchEsnemble, then there would exist $W$, $r$, and $s$ such that $A = W \circ rs^\top$ and $B=W$.
  Then, $Q = rs^\top$ has rank 1, a contradiction.
  Finally, note that $\mathcal{W}_\text{BE}$ depends on a total of $mn + k(m+n)$ parameters.
  Instead, $\mathcal{W}_\text{DE}$ has $k mn$ parameters.
  For $k, m, n>3$, $\mathcal{W}_\text{BE}$ is a lower-dimensional subset of $\mathcal{W}_\text{DE}$ and has therefore measure zero in $\mathcal{W}_\text{DE}$.
\end{proof}

\section{Metrics}
\label{appendix:metrics}

\paragraph{Predictive performance.}
We report accuracy (higher is better) and negative log-likelihood (NLL; lower is better):
\[
\mathrm{NLL}(x,y) = -\log p_\theta(y\mid x).
\]
NLL rewards high probability for correct outcomes and penalizes uncertainty.
However, it punishes \emph{overconfident wrong predictions} more severely than the Brier score due to its logarithmic form.

\paragraph{Calibration.}
Expected Calibration Error (ECE; lower is better) partitions predictions into \(B\) confidence bins and measures the weighted average gap between empirical accuracy and mean confidence:
\[
\mathrm{ECE} = \sum_{b=1}^{B} \frac{|\mathcal{B}_b|}{N}\,\bigl|\mathrm{acc}(\mathcal{B}_b)-\mathrm{conf}(\mathcal{B}_b)\bigr|.
\]
We use \(B{=}15\) equal-width bins.

\paragraph{Uncertainty proxies.}
Let \(\bar{p}(y\mid x)=\tfrac{1}{k}\sum_{i=1}^k p_i(y\mid x)\) be the ensemble mean predictive distribution.
Predictive entropy (higher indicates more total uncertainty):
\[
H(x)= -\sum_c \bar{p}(y{=}c\mid x)\,\log \bar{p}(y{=}c\mid x).
\]
Jensen–Shannon divergence (JSD; higher indicates more epistemic disagreement):
\[
\mathrm{JSD}(x)= H(x) - \frac{1}{k}\sum_{i=1}^k H_i(x),
\]
where $H_i$ denotes the predictive entropy of the $i$-th ensemble member.

\paragraph{OOD detection.}
Beyond moderate shifts in input quality, neural networks are often deployed in settings where the input domain may deviate substantially from anything seen during training such that the entire notion of prediction becomes ill-posed.
In such cases, even well-calibrated models may fail catastrophically if they do not recognize that the input lies outside their domain of competence.
To address this, models must be equipped with mechanisms to identify OOD inputs and abstain from confident predictions \citep{amodei2016concrete}.

One simple and widely used approach for OOD detection is to use the maximum softmax probability (MSP) of a classifier as a confidence score \citep{hendrycks2016baseline}.
If this score falls below a certain threshold, the input is flagged as OOD.
This transforms the task into a binary classification problem: 
determining whether a given input is from the training distribution (ID) or not (OOD).

To evaluate OOD detection performance, we define the confidence function $s(x) \in \mathbb{R}$ as the MSP.
An input $x$ is rejected (i.e., classified as OOD) if $s(x) < \tau$ for a threshold $\tau$:
$$
    \text{Reject}(x) = \mathbb{I}(s(x) < \tau).
$$
This setup allows us to apply standard binary classification metrics, which we describe below.

\textit{AUROC} (area under ROC; higher is better) is a threshold-free ranking metric defined by the ROC curve \((\mathrm{TPR}(\tau), \mathrm{FPR}(\tau))\) as the decision threshold \(\tau\) varies:
\[
\mathrm{AUROC}=\int_0^1 \mathrm{TPR}(\mathrm{FPR})\, d\,\mathrm{FPR}.
\]
An equivalent interpretation is the probability that a randomly chosen ID input $x^+$ receives a higher confidence score than a randomly chosen OOD input $x^-$:
$$
    \text{AUROC} = \mathbb{P}(s(x^+) > s(x^-)).
$$

\textit{AUPR} (area under precision–recall; higher is better) summarizes the PR curve as recall varies:
\[
\mathrm{Precision}=\frac{\mathrm{TP}}{\mathrm{TP}+\mathrm{FP}},\quad
\mathrm{Recall}=\frac{\mathrm{TP}}{\mathrm{TP}+\mathrm{FN}},\quad
\mathrm{AUPR}=\int_0^1 P(R)\, dR.
\]
It quantifies how many of the samples predicted as OOD are actually OOD across all thresholds $R$.
Since AUPR is more sensitive than AUROC when positives are rare \cite{davis2006relationship} it is useful for OOD scenarios where OOD samples are rare.

\textit{FPR95} (lower is better) is the false-positive rate at \(\mathrm{TPR}=0.95\).
It measures the proportion of OOD samples misclassified as ID when the true positive rate (TPR) for ID samples is set to 95:
\[
\mathrm{FPR95}=\mathrm{FPR}\ \text{at}\ \mathrm{TPR}=0.95.
\]
As such it captures worst-case performance under high recall.

\section{Experimental Setup}
\label{appendix:experimental-details}
\paragraph{General procedure.}
Unless specified otherwise, all models are trained using the standard cross-entropy loss applied to individual member predictions (not to the ensemble average).
For MC Dropout, we follow the common practice of training the model as a deterministic network with dropout enabled only at test time.
Deep Ensemble and BatchEnsemble predictors are treated as independent models: each member’s loss is computed and backpropagated separately.
For BatchEnsemble, we use the formulation described in Appendix~\ref{appendix:batchensemble-vectorization}, repeating each training sample $k$ times per batch and replicating the scaling vectors $\mathcal{B}$ times so that every predictor observes the entire dataset in each epoch.
All experiments are run with random seeds $\{42, 69, 123, 456, 789\}$, and confidence intervals are computed using bootstrapping with the Student-$t$ distribution (4~d.o.f.).
All code and configuration files are publicly available at this \textcolor{blue}{\href{https://osf.io/a8yt6/overview?view_only=13df5f17c6764b9aa0b864d387a1cc40}{repository}}.

\paragraph{Implementation details.}
Unless stated otherwise, we adopt the default hyperparameter configurations from \texttt{torch-uncertainty}, as used by \citet{laurent2022packed}.
Dropout layers are applied after each linear layer and before the activation in MLPs, and within ResNet blocks after the first convolution and before the activation, matching the default implementations in \texttt{torch-uncertainty}.
All convolutional and linear layers use He (Kaiming) initialization (PyTorch defaults for ReLU networks).
BatchEnsemble scaling vectors $r_i$ and $s_i$ are initialized from $\mathcal{N}(1,\,0.5^2)$, following the original implementation.

\subsection{CIFAR-10}
\label{appendix:experimental-details-cifar-10}

We follow the evaluation protocol of \citet{ovadia2019predictiveUncertaintyDataShift}.

\begin{itemize}
    \item \textbf{Models:} Single model, MC Dropout (rate 0.1, four forward passes), Deep Ensemble ($k{=}4$), and BatchEnsemble ($k{=}4$).
    \item \textbf{Architecture:} All models use ResNet-18 \citep{he2016resnet}.
    \item \textbf{Datasets:} Training on CIFAR-10 \citep{krizhevsky2009cifar}; evaluation on:
    \begin{enumerate}
        \item CIFAR-10 test set (in-distribution), 
        \item CIFAR-10C \citep{hendrycks2019benchmarkingRobustness}, corruption severity 5 (distribution shift),
        \item SVHN \citep{netzer2011svhn} (OOD).
    \end{enumerate}
    \item \textbf{Training:} 75 epochs with SGD, Nesterov momentum 0.9, weight decay $5\times10^{-4}$, batch size 128, and an initial learning rate of 0.05 decayed by 0.1 at epochs 25 and 50.
    \item \textbf{Reporting:} Mean and 95\% CIs (Student-$t$, 4~d.o.f.); AUROC/AUPR micro-averaged across OOD samples.
    \item \textbf{Hardware:} Single NVIDIA RTX~A5000 GPU.
\end{itemize}

To ensure comparability between methods, we use four ensemble members for all methods and interpret MC Dropout as an ensemble of four stochastic forward passes.
Due to computational constraints, we used smaller ensembles than those in \citet{ovadia2019predictiveUncertaintyDataShift}, but they report diminishing returns beyond five members and beyond four MC samples.
Although \citet{ovadia2019predictiveUncertaintyDataShift} used slightly larger models and longer training (ResNet-20, 200 epochs), our single-model, Deep Ensemble, and MC Dropout results closely reproduce their ID and shifted results.

\subsection{MNIST}
\label{appendix:experimental-details-mnist}
\begin{itemize}
    \item \textbf{Models:} Single model, MC Dropout (rate 0.05, four forward passes), Deep Ensemble ($k{=}4$), and BatchEnsemble ($k{=}4$).
    \item \textbf{Architecture:} 3-layer MLP with hidden dimensions (64, 32, 16) and ReLU activation.
    \item \textbf{Datasets:} MNIST \citep{6296535} for training and evaluation on:
    \begin{enumerate}
        \item MNIST test set (in-distribution),
        \item MNIST-C \citep{mu2019mnistc} (distribution shift),
        \item FashionMNIST \citep{xiao2017fashion} (OOD).
    \end{enumerate}
    \item \textbf{Training:} Up to 75 epochs with early stopping (patience = 10, threshold = 0.001 val-acc), Adam optimizer \citep{kingma2014adam} with $\beta_1{=}0.9$, $\beta_2{=}0.999$, $\epsilon{=}10^{-8}$, weight decay $10^{-4}$, batch size 128, learning rate 0.001 decayed by 0.1 at epochs 25 and 50.
    \item \textbf{Hardware:} Single NVIDIA A40 GPU.
\end{itemize}

\section{Computational Cost and Training Time}
\label{appendix:computational-cost}
We report the number of trainable parameters and average training time (including evaluation on validation and test sets) for each method on ResNet-18 trained on CIFAR-10, using the experimental setup described in Appendix~\ref{appendix:experimental-details-cifar-10}.
Training BatchEnsemble as described such that each ensemble member sees all training data (see Appendix~\ref{appendix:batchensemble-vectorization}) leads to an increase in training duration approaching the training time of Deep Ensemble. 

\begin{table}[h!]
    \centering
    \small
    \caption{Comparison of Computational Efficiency}
    \begin{tabular}{c|ccccc}
        \hline
        Method & Single & MC Dropout & Deep Ensemble & BatchEnsemble \\ \hline
        \#Params & 11.17 Mio. & 11.17 Mio. & 44.71 Mio. & 11.22 Mio. \\
        Training Time & 15min & 18min & 1h15min & 1h1min \\ \hline
    \end{tabular}
    \label{tab:computational-cost}
\end{table}

\section{Channel Rescaling vs. Rank-1 Kernel Perturbation}
\label{appendix:channel-rescaling-vs-kernel-perturbation}

In the original BatchEnsemble implementation used in this work, convolution layers are perturbed by rescaling the input and output channels using two learned vectors $r \in \mathbb{R}^{C_{\text{in}}}$ and $s \in \mathbb{R}^{C_{\text{out}}}$, respectively.
These scaling factors are applied multiplicatively to the input and output of a shared convolution kernel and are broadcast across the spatial dimensions.
This approach modifies the activations rather than the convolution weights themselves which may be less expressive than applying rank-1 perturbations kernel-wise, as done in linear layers

In this appendix, we explore an alternative formulation where perturbations are applied directly to the convolution kernels using the same rank-1 factorization as for linear layers.
Following Equation~\eqref{eq:batch-ensemble}) each ensemble member's convolution kernel $W_i \in \mathbb{R}^{C_{\text{out}} \times C_{\text{in}} \times K \times K}$ is constructed via element-wise multiplication of the shared kernel $W$ with the outer product of learned $r_i \in \mathbb{R}^{C_{\text{out}} \times C_{\text{in}} \times K}$ and $s_i \in \mathbb{R}^{C_{\text{out}} \times C_{\text{in}} \times K}$ taken over the last dimensions of size $K$.

We compare channel-wise (BatchEnsemble (C)) and kernel-wise (BatchEnsemble (K)) perturbations on CIFAR-10, CIFAR-10C (severity 5), and SVHN.
While the kernel-wise variant improves calibration and OOD detection slightly, it performs worse in terms of accuracy even compared to the already underperforming channel-wise variant.
This is likely driven by inflated entropy in the individual members’ predictions: the ensemble’s predictive entropy (H) increases, yet JSD stays low, suggesting a rise in aleatoric rather than epistemic uncertainty.
Despite the added parameters and computational cost\footnote{
    For $3 \times 3$ convolutions and rank-1 perturbations, kernel-wise perturbations introduce 6 parameters per filter, leading to $6 \cdot C_{\text{in}} C_{\text{out}}$ parameters per layer, compared to just $C_{\text{in}} + C_{\text{out}}$ for channel-wise scaling.
}, kernel-wise perturbations still fall short of approximating the diversity and performance of Deep Ensembles or MC Dropout.

\begin{table}[!h]
    \centering
    \footnotesize
    \caption[
        BE (C) vs. (K): In-Distribution Benchmark (CIFAR-10)
    ]{
        In-Distribution Benchmark (CIFAR-10)
    }
    \begin{tabular}{c|ccc|cc}
        \hline
        \textbf{Method} & \textbf{Acc↑} & \textbf{NLL↓} & \textbf{ECE↓} & \textbf{H} & \textbf{JSD} \\ \hline
        Standard & \cellcolor{blue!10}0.941 ±0.004 & 0.237 ±0.003 & 0.034 ±0.003 & 0.073 ±0.005 & N/A \\ 
        MC Dropout & 0.936 ±0.001 & \cellcolor{blue!10}0.210 ±0.007 & \cellcolor{blue!10}0.021 ±0.002 & 0.120 ±0.005 & 0.022 ±0.001 \\ 
        Deep Ensemble & \cellcolor{blue!30}0.952 ±0.001 & \cellcolor{blue!30}0.152 ±0.003 & \cellcolor{blue!30}0.007 ±0.002 & 0.136 ±0.001 & 0.037 ±0.001 \\ 
        BatchEnsemble (C) & 0.938 ±0.001 & 0.230 ±0.006 & 0.032 ±0.001 & 0.088 ±0.003 & 0.002 ±0.000 \\
        BatchEnsemble (K) & 0.931 ±0.002 & 0.231 ±0.010 & 0.025 ±0.002 & 0.124 ±0.003 & 0.010 ±0.001 \\
        \hline
    \end{tabular}
    \label{tab:channel-vs-kernel-ind}
\end{table}

\begin{table}[!h]
    \centering
    \footnotesize
    \caption{
        Distribution Shift Benchmark (CIFAR-10C, shift intensity 5)
    }
    \begin{tabular}{c|ccc|cc}
        \hline
        \textbf{Name} & \textbf{Acc↑} & \textbf{NLL↓} & \textbf{ECE↓} & \textbf{H} & \textbf{JSD} \\ \hline
        Standard & 0.558 ±0.008 & 2.545 ±0.142 & 0.323 ±0.014 & 0.327 ±0.026 & N/A \\ 
        MC Dropout &\cellcolor{blue!30}0.578 ±0.002 & \cellcolor{blue!10}2.054 ±0.084 & \cellcolor{blue!10}0.251 ±0.009 & 0.449 ±0.031 & 0.088 ±0.005 \\ 
        Deep Ensemble & \cellcolor{blue!10}0.575 ±0.006 & \cellcolor{blue!30}1.682 ±0.045 & \cellcolor{blue!30}0.206 ±0.009 & 0.593 ±0.018 & 0.176 ±0.005 \\ 
        BatchEnsemble (C) & 0.547 ±0.015 & 2.357 ±0.164 & 0.308 ±0.012 & 0.400 ±0.011 & 0.008 ±0.002 \\
        BatchEnsemble (K) & 0.521 ±0.008 & 2.266 ±0.110 & 0.275 ±0.012 & 0.544 ±0.023 & 0.039 ±0.004 \\
        \hline
    \end{tabular}
    \label{tab:channel-vs-kernel-shifted}
\end{table}

\begin{table}[!h]
    \centering
    \footnotesize
    \caption[
        BE (C) vs. (K): OOD Benchmark (SVHN)
    ]{
        OOD Benchmark (SVHN)
    }
    \begin{tabular}{c|ccc|cc}
        \hline
        \textbf{Name} & \textbf{AUPR↑} & \textbf{AUROC↑} & \textbf{FPR95↓} & \textbf{H} & \textbf{JSD} \\ \hline
        Standard & 0.938 ±0.017 & 0.893 ±0.024 & 0.318 ±0.071 & 0.480 ±0.096 & N/A \\ 
        MC Dropout & 0.932 ±0.008 & 0.884 ±0.009 & 0.324 ±0.019 & 0.608 ±0.056 & 0.122 ±0.016 \\ 
        Deep Ensemble & \cellcolor{blue!30}0.953 ±0.006 & \cellcolor{blue!30}0.924 ±0.011 & \cellcolor{blue!30}0.199 ±0.024 & 0.829 ±0.067 & 0.277 ±0.032\\ 
        BatchEnsemble (C) & 0.932 ±0.015 & 0.888 ±0.020 & 0.297 ±0.061 & 0.495 ±0.100 & 0.010 ±0.006\\
        BatchEnsemble (K) & \cellcolor{blue!10}0.944 ±0.008 & \cellcolor{blue!10}0.911 ±0.013 & \cellcolor{blue!10}0.231 ±0.025 & 0.692 ±0.064 & 0.043 ±0.008 \\
        \hline
    \end{tabular}
    \label{tab:channel-vs-kernel-ood}
\end{table}

\end{document}